\documentclass{article} % For LaTeX2e
\usepackage{nips13submit_e,times}
\usepackage{hyperref}
\usepackage{url}
\usepackage{graphicx}
\usepackage{subcaption}
\usepackage{amsmath}
\usepackage{amsthm}
\usepackage{dsfont}
\usepackage{float}
\usepackage[ruled]{algorithm2e}

\graphicspath{{figures/}}
\newtheorem{prop}{Proposition}

\title{Faster Clustering via Non-Backtracking Random Walks}

\author{
Brian Rappaport\thanks{First and second authors contributed equally. Research supported by the National Science Foundation Research Experience for Undergraduates (NSF-REU).} \\
Department of ECE\\
Tufts University\\
\texttt{brian.rappaport@tufts.edu} \\
\And
Anuththari Gamage$^*$ \\
Department of ECE\\
Tufts University\\
\texttt{anuththari.gamarallage@tufts.edu} \\
\AND
Shuchin Aeron \\
Department of ECE\\
Tufts University\\
\texttt{shuchin@ece.tufts.edu} 
}

\nipsfinalcopy % Uncomment for camera-ready version

\begin{document}

\maketitle

\begin{abstract}
This paper presents VEC-NBT, a variation on the unsupervised graph clustering technique VEC, which improves upon the performance of the original algorithm significantly for sparse graphs. VEC employs a novel application of the state-of-the-art word2vec model to embed a graph in Euclidean space via random walks on the nodes of the graph. In VEC-NBT, we modify the original algorithm to use a non-backtracking random walk instead of the normal backtracking random walk used in VEC. We introduce a modification to a non-backtracking random walk, which we call a begrudgingly-backtracking random walk, and show empirically that using this model of random walks for VEC-NBT requires shorter walks on the graph to obtain results with comparable or greater accuracy than VEC, especially for sparser graphs.
\end{abstract}

\section{Introduction}

The word2vec algorithm \cite{word2vec} has become one of the most commonly used models for natural language processing, being both faster and more accurate than most other choices for embedding words into Euclidean space. In the word2vec model, an input corpus of sentences is used to create a co-occurrence matrix of words in the vocabulary, and the word vectors are optimized using a cost function related to the adjacencies of words to one another, a process known as Skip-Gram with Negative Sampling (SGNS). The resulting vectors can be used to perform various natural language processing tasks such as analogy prediction and sentiment analysis.

This process is not limited to natural language processing. Identifying communities in a graph requires the definition of some measure of similarity between nodes of a graph. The VEC algorithm proposed by Ding et al. \cite{NodeEmbed} quantifies this similarity by considering the nodes of a graph as words contained in sentences formed by random walks on the graph. Once this word-sentence representation is obtained, identifying graph communities is analogous to finding semantic or syntactic similarity between words in a language, since these similarities are often defined by the frequency of word pairs or groups in sentences. Starting at each node in the graph, VEC performs several random walks of fixed length. Using the skip-gram method, where two words are considered adjacent if they are within a certain distance of each other in one or more of the sentences, each sentence adds several node-pairs to the co-occurrence matrix of the list of nodes. Word2vec can be used to convert the random walks into node embeddings in Euclidean space, which is in effect a factorization of the co-occurrence matrix, as explained in \cite{LevyGoldberg}.

The resulting vectors can then be clustered using simple techniques such as k-means clustering. The clusters identified by this method have been shown to improve considerably on those generated by classic algorithms such as spectral clustering and acyclic belief propagation \cite{NodeEmbed}. However, this algorithm performs less effectively on extremely sparse graphs: if the average degree of the nodes is lower than 3, the algorithm cannot reliably determine the clusters. Unfortunately, many real-world graphs have this feature, such as protein-protein interaction graphs, our motivating example.

We attempt to address these problems with the novel concept of using non-backtracking random walks to form the sentences, which encourages more homogeneous clusters. A non-backtracking random walk is a random walk which does not return to the node which it visited in the previous step. Intuitively, this idea is logical: backtracking does not add any new information to the algorithm, so reducing the number of repeated edges should produce better embeddings, which in turn give better clusters. Several groups have conducted theoretical explorations of non-backtracking random walks, which give a more formal justification to this claim \cite{Redemption, Alon, NBT-Ihara}. We discuss some of the key ideas of these papers as well as show experimentally the clear improvement in performance which results from using non-backtracking random walks.

%% \subsection{Double-blind reviewing}

%% This year we are doing double-blind reviewing: the reviewers will not know 
%% who the authors of the paper are. For submission, the NIPS style file will 
%% automatically anonymize the author list at the beginning of the paper.

%% Please write your paper in such a way to preserve anonymity. Refer to
%% previous work by the author(s) in the third person, rather than first
%% person. Do not provide Web links to supporting material at an identifiable
%% web site.

%%\subsection{Electronic submission}
%%
%% \textbf{THE SUBMISSION DEADLINE IS MAY 31st, 2013. SUBMISSIONS MUST BE LOGGED BY
%% 23:00, MAY 31st, 2013, UNIVERSAL TIME}

%% You must enter your submission in the electronic submission form available at
%% the NIPS website listed above. You will be asked to enter paper title, name of
%% all authors, keyword(s), and data about the contact
%% author (name, full address, telephone, fax, and email). You will need to
%% upload an electronic (postscript or pdf) version of your paper.

%% You can upload more than one version of your paper, until the
%% submission deadline. We strongly recommended uploading your paper in
%% advance of the deadline, so you can avoid last-minute server congestion.
%%
%% Note that your submission is only valid if you get an e-mail
%% confirmation from the server. If you do not get such an e-mail, please
%% try uploading again. 

\section{Description}

\subsection{Notation}

Let $G = (V,E)$ be a graph with vertex set $V$ and edge set $E$, and $|V| = n, |E| = m$. The \emph{adjacency matrix} $A$ of $G$ is defined as the $n\times n$ matrix with $a_{u,v} = 1$ if and only if $(u,v)\in E$, and the \emph{degree matrix} $D$ of $G$ is the $n \times n$ diagonal matrix indexed by $v$ with each diagonal element equal to the degree of vertex $v$. A \emph{random walk} on $G$ is defined as a sequence of vertices $(v_0,v_1,...,v_k)$, each connected by an edge in $E$, where at each step the next vertex is chosen randomly from those neighboring the current step with equal probability. A random walk on $G$ is also a Markov process, where the \emph{transition probability matrix} $P = D^{-1}A$, which gives $x_kP = x_{k+1}$ for starting distribution $x_0$ and $k>0$. This recurrence gives the closed form expression $x_t = x_0P^t$. We can also determine the \emph{stationary distribution} $\pi(v)$, defined as the distribution such that $\pi P = \pi$, as $dv/vol(G)$, the degrees of the nodes divided by the total number of edges. For any non-bipartite connected graph $G$, the stationary distribution is the limit $\displaystyle\lim_{t\rightarrow\infty}{x_0P^t}$ \cite{Lovasz}.

The \emph{graph Laplacian} $L$ is defined as $D - A,$ and is often represented as one of two normalized forms, $L_{sym} = I - D^{-1/2}AD^{-1/2}$ and $L_{rw} = I - D^{-1}A$: $L_{rw}$ is closely related to $P$ (it is in fact $I-P$), and $L_{sym}$ is a symmetric matrix which is similar to $L$ and so has the same eigenvalues. The eigenvalues of $L_{rw}$ are also closely related to those of $L$ \cite{Luxburg}. The \emph{mixing rate} of a graph, defined as 
$$\rho = \limsup_{t\rightarrow\infty}\max_{u,v}\left|P^{t}(u,v)-\pi(v)\right|^{1/t},$$ 
defines how quickly a Markov chain with transition probabilities $P$ reduces to the stationary distribution $\pi(v)$, where a lower mixing rate indicates a faster mixing time. The mixing rate is intimately connected to the second eigenvalue of the graph Laplacian, as detailed in \cite{Lovasz}. A Markov chain is \emph{irreducible} if every node can be accessed from every other node; it is \emph{aperiodic} if every node can have a cycle of any length (for instance, a bipartite graph is not aperiodic since all cycles for any node will have even length). A Markov chain is \emph{ergodic} if it is aperiodic and irreducible. 

A \emph{non-backtracking random walk} is defined as a random walk that chooses its next step from all neighbors except the one it visited in the previous step. A \emph{begrudgingly-backtracking random walk} is the same, with the added condition that if the only choice of edge is the one visited previously, then it will resort to backtracking for that edge. Begrudgingly-backtracking random walks handle several problems caused by non-backtracking random walks especially on sparse graphs: in a non-backtracking random walk, walking to a dangling node (a node with only one connecting edge) forces the walk to end there, resulting in an artificially decreased length of the random walk. In addition, the dangling nodes will be weighted less heavily than they should be, since they will only ever be visited once per walk.

\subsection{Algorithm}

The graphs used to measure the performance of VEC-NBT are synthesized using the Stochastic Block Model (SBM), a canonical graph model used for community detection. SBM builds off of the classical Erdos-Renyi $G(n,p)$ random graph model, where each edge of a graph with $n$ nodes is formed with probability $p$. In SBM, specifically the planted partition model, $G(n,p)$ is additionally given $k$ clusters and each node is added to one of the clusters with equal probability. Edges are formed within the cluster with probability $a$ while inter-cluster edges are formed with probability $b$, with $b$ lower than $a$ in order to form an assortative graph. We have elected to use graphs with constant scaling, $b = \lambda a$. In our model, the probability of an intra-cluster edge being formed is $Q_n(k,k) = a = \frac{c}{n}$ while that of an inter-cluster edge being formed is $Q_n(k,k^\prime) = b = \frac{c(1-\lambda)}{n}$. $c$ is the average degree of nodes within the cluster and determines the sparsity of the graph. $\lambda$ determines how connected the various clusters are: $\lambda=1$ would imply completely disjoint clusters, while $\lambda=0$ draws no distinction between clusters.  Thus, our model can be represented as $G(n,k,c,\lambda)$, which fully determines the graph.

To generate a Euclidean embedding of the nodes of the graph, VEC performs $r = 10$ random walks per node, each with length $l = 60$. Since the graphs are unweighted, each neighbor is equally likely to be chosen. Each random walk is converted into a sentence by counting the number of pairwise co-occurrences for all node pairs: for the co-occurrence matrix $W$, $W_{i,j}$ is the total number of times node $j$ occurs within $w=8$ places of node $i$ in all random walks (this defines a skip-gram with window size 8). Note that an isolated node (a node with a degree of 0) will not be part of any sentence, since the random walk will not reach it and random walks starting there will not be included in the corpus. This observation makes logical sense because a node with no connections cannot be said to be in any cluster, since there is no information about the node at all.

Once the sentences have been formed, an existing implementation of a popular word embedding algorithm, namely, word2vec \cite{word2vec}, is used to convert the sentences into vectors. word2vec takes the corpus of random walks and embeds them into $d$-dimensional space by means of a stochastic gradient descent algorithm. An embedding dimension of $d=50$ is used in our tests to correspond to the original parameters used for VEC. Finally, the embeddings are clustered using a standard k-means clustering scheme.

\begin{algorithm}
    \caption{VEC-NBT}
    \SetKwInOut{Input}{Input}
    \SetKwInOut{Output}{Output}

    \Input{Graph $G$, Number of clusters $k$}
    \Output{Estimated label vector $y$}
    
    \For{\emph{node }$v \in V,\; t \in \{1 \dots r\}$}
        {$S_{v,t} := $ begrudgingly-backtracking random walk of length $l$ starting at $v$}
    $\{U_i\}_{i=1}^{n} := $ embedding vector for each node generated via word2vec using $S$\\
    $y := $ clusters given by K-Means on rows of $U$
\end{algorithm}

VEC is shown to have consistently performed better than standard community detection methods, such as spectral clustering and acyclic belief propagation, both in accuracy and robustness to random initialization of the graph \cite{NodeEmbed}. However, accuracy is still lower than desirable for very sparse graphs. In addition, although 60 is not an excessively long random walk, reducing it would speed up the algorithm.

We propose replacing the simple random walk in VEC with a begrudgingly-backtracking random walk, through which we find that both the accuracy and the runtime of VEC can be improved while using shorter random walks. By removing the possibility of revisiting an already encountered node, the begrudgingly-backtracking random walk diffuses over the nodes of the graph faster than a backtracking random walk due to its faster mixing rate \cite{Alon,NBT-Ihara}. Thus, a shorter begrudgingly-backtracking random walk can identify the community structure of sparse graphs better than a simple random walk of greater length. This is a classic example of exploration vs. exploitation in machine learning and statistics.

\section{Theory}

VEC-NBT produces significantly better clustering than VEC consistently across a range of graph sparsity levels and number of nodes. This improvement in performance can be attributed to the faster mixing rate of a begrudgingly-backtracking random walk compared to a simple random walk.

Consider a simple random walk on a non-bipartite graph $G$  where each node $u$ has degree $d_u$. This is a first-order Markov chain with a transition probability matrix $P$ as follows \cite{Lovasz}.
\[
    P(u,v) = 
    \begin{cases} 
        \displaystyle\frac{1}{d_u} & \text{if } uv \in E \\
        0 & \text{otherwise} 
    \end{cases}
\]

As stated in Section 2.1, the stationary distribution for a random walk on $G$ is given by $\displaystyle\pi(v) = \frac{d_v}{vol(G)}$. However, a non-backtracking random walk on a graph $G$ is a second-order Markov chain. We impose the additional constraint that $G$ has a minimum degree of 2 for all nodes. To convert this random walk into a first-order Markov chain, the transition probability matrix $\tilde P$ is defined on the directed edge set of the graph instead of the vertex set \cite{NBT-Ihara}. $\tilde P$ is a $2m \times 2m$ matrix with $\tilde P((u,v), (x,y))$ representing the transition probability between edge $(u,v)$ to edge $(x,y)$ such that
\[\tilde P((u,v), (x,y)) = 
    \begin{cases} 
        \displaystyle\frac{1}{d_v-1} & \text{if } v = x, y \neq u\\
        0 & \text{otherwise}
    \end{cases}
\]

Note that $\tilde P$ is doubly stochastic. It is proven in \cite{NBT-Ihara} that since $\tilde P$ is irreducible and aperiodic, the non-backtracking random walk converges to the stationary distribution
$$\tilde \pi = \frac{\mathds{1}}{vol(G)}$$
where $\mathds{1}$ is the unit vector of length $2m$.

As discussed in \cite{NBT-Ihara}, the mixing rate of the backtracking and the non-backtracking random walk is equal to the second largest eigenvalues of their respective transition probability matrices, $P$ and $\tilde P$. Thus, if $\rho$ and $\tilde \rho$ are the mixing times of the backtracking and non-backtracking random walks respectively and $\lambda_2$ is the second largest eigenvalue of the adjacency matrix of a $d$-regular graph $G$, we have 
$$\rho = \frac{\lambda_2}{d}$$
and
$$\tilde \rho = \frac{\lambda_2 + \sqrt{\lambda_2^2 - 4(d-1)}}{2(d-1)}.$$

Furthermore, \cite{NBT-Ihara, Alon} prove that the non-backtracking random walk has a faster mixing rate than a backtracking random walk, yielding the following bounds:

\begin{flushleft}
For $2\sqrt{d-1} \leq \lambda_2 \leq d:$
\end{flushleft}
$$\frac{d}{2(d-1)} \leq \frac{\tilde \rho}{\rho} \leq 1$$

\begin{flushleft}
For $\lambda_2 \leq 2\sqrt{d-1} \text{ and } d = n^{o(n)}:$
\end{flushleft}
$$\frac{\tilde \rho}{\rho} = \frac{d}{2(d-1)} + o(1)$$\\

In VEC-NBT, we use a begrudgingly-backtracking random walk, a variation of the non-backtracking random walk. Let $\hat P$ be a $2m \times 2m$ matrix defining the transition probabilities of the begrudgingly-backtracking random walk on the edge set of a graph $G$ such that
\[\hat P((u,v), (x,y)) = 
    \begin{cases} 
        \displaystyle\frac{1}{d_v-1} & \text{if } v = x, y \neq u\\
        1 & \text{if } v = x, y = u, d_v = 1 \\
        0 & \text{otherwise}
    \end{cases}
\]

\begin{prop}
$\hat P$ is doubly stochastic.
\end{prop}

\begin{proof}
Since $\tilde P$ has been shown to be doubly stochastic, we only need consider those elements that would change from $\tilde P$ to $\hat P$, and these are exactly the rows and columns corresponding to any dangling nodes. Let $x$ be a node of $G$ with degree 1, connected only to node $y$; then the single row going to $x$, $(y,x)$, has only the element returning to $y$, $(x,y)$, which has a weight of 1; similarly, the single column coming from $x$ has only the element that sent it there. 
\end{proof}

Hence, $\tilde P$ is the same as $\hat P$, except in the rows and columns of degree-1 nodes: where $\tilde P$ has rows and columns of zeros, $\hat P$ has a 1 where the edge can be included. Since $\tilde P$ is only defined on graphs where the minimal degree is greater than 1, $\tilde P$ is still doubly stochastic, and $\hat P$ extends the definition to graphs with singlet connectivity. In this way we can relax the requirement that all elements of $G$ have minimal degree 2.

We can also show that $\hat P$ is also irreducible and aperiodic, provided $\hat P$ has a single connected component and at least one element with degree greater than 1: clearly if $\tilde P$ is irreducible, extending the graph to include single edges that can be traveled down and back will not make the chain reducible; and it can be shown that an irreducible Markov chain with at least one aperiodic node is aperiodic and $\tilde P$ is aperiodic, so adding more nodes will not change that result. Therefore, by the same argument made for $\tilde \pi$, we can see that the stationary distribution for the begrudgingly-backtracking random walk is
$$\hat \pi = \frac{\mathds{1}}{vol(G)}$$

Thus, we can think of the begrudgingly-backtracking random walk as a variant of the non-backtracking random walk which only requires the graph $G$ to have $d_u \geq 1$ for each node $u$. Given the similarities between the two random walks, we hypothesize that the begrudgingly-backtracking random walk has a mixing rate that is equal or similar to that of the non-backtracking random walk, which explains the fast convergence and greater accuracy observed in our experiments. Furthermore, we suspect that even though the faster mixing rate of the non-backtracking random walk is proven only for a $d$-regular graph, it seems to hold in our experiments since the parameter $c$ used to generate the SBM graphs ensures that each node has a constant average degree. These two points are yet to be fully explored and formalized. 

\section{Numerical Results}

We compare the performance of VEC and VEC-NBT on SBM graphs generated using the parameters given through two metrics: Correct Classification Rate (CCR) and Normalized Mutual Information (NMI). For random walks and embedding, VEC-NBT uses the same parameters used by VEC with the exception of the length of the random walk ($l=5, 10$) and the window size ($w=5$) and twice the number of random walks ($r=20$). Here, we show empirically that VEC-NBT consistently achieves better accuracy than VEC for sparser graphs (low values of $c$) and comparable accuracy to VEC at higher sparsity levels.

CCR is defined as the number of correctly classified points divided by the total number of nodes. To ensure the calculated clustering matches the ground truth, we use a linear sum assignment to match the cluster assignments to the original labels. Because of this, CCR is defined (at least for 2 clusters) only between 0.5 and 1, since a measured CCR of 0 would indicate that every 1 was labeled as a 2 and vice versa - which is in fact a perfect clustering. In our graphs we have plotted CCR as a percentage between $1/K$ and 1 where $K$ is the number of clusters.

NMI is defined as the mutual information $I(X;Y)$ normalized by the square root of the entropies $H(X)$ and $H(Y)$:
$$NMI(X,Y) := \frac{I(X;Y)}{\sqrt{H(X)H(Y)}} = \frac{H(X)+H(Y)-H(X,Y)}{\sqrt{H(X)H(Y)}},$$
a more technical metric measuring the information content of the resulting labels.

Figures are shown with the original VEC algorithm (``BT") with solid lines and our new algorithm (``NBT") with dashed lines. CCR and NMI are shown for each algorithm on each plot. Note that red points correspond to CCR measurements, on a 50-100\% scale, and blue points correspond to NMI measurements, between 0 and 1. NMI tends to be a more accurate indicator of performance.

Unless otherwise specified, the x-axis is the sparsity of the graph, varying from 2 to 20; the number of clusters is 2; the graph has 10000 nodes; and the walks are 10 steps long.

\begin{figure}[H] %% Figure 1
    \centering
    \begin{subfigure}{0.48\textwidth}
        \includegraphics[width=\linewidth]{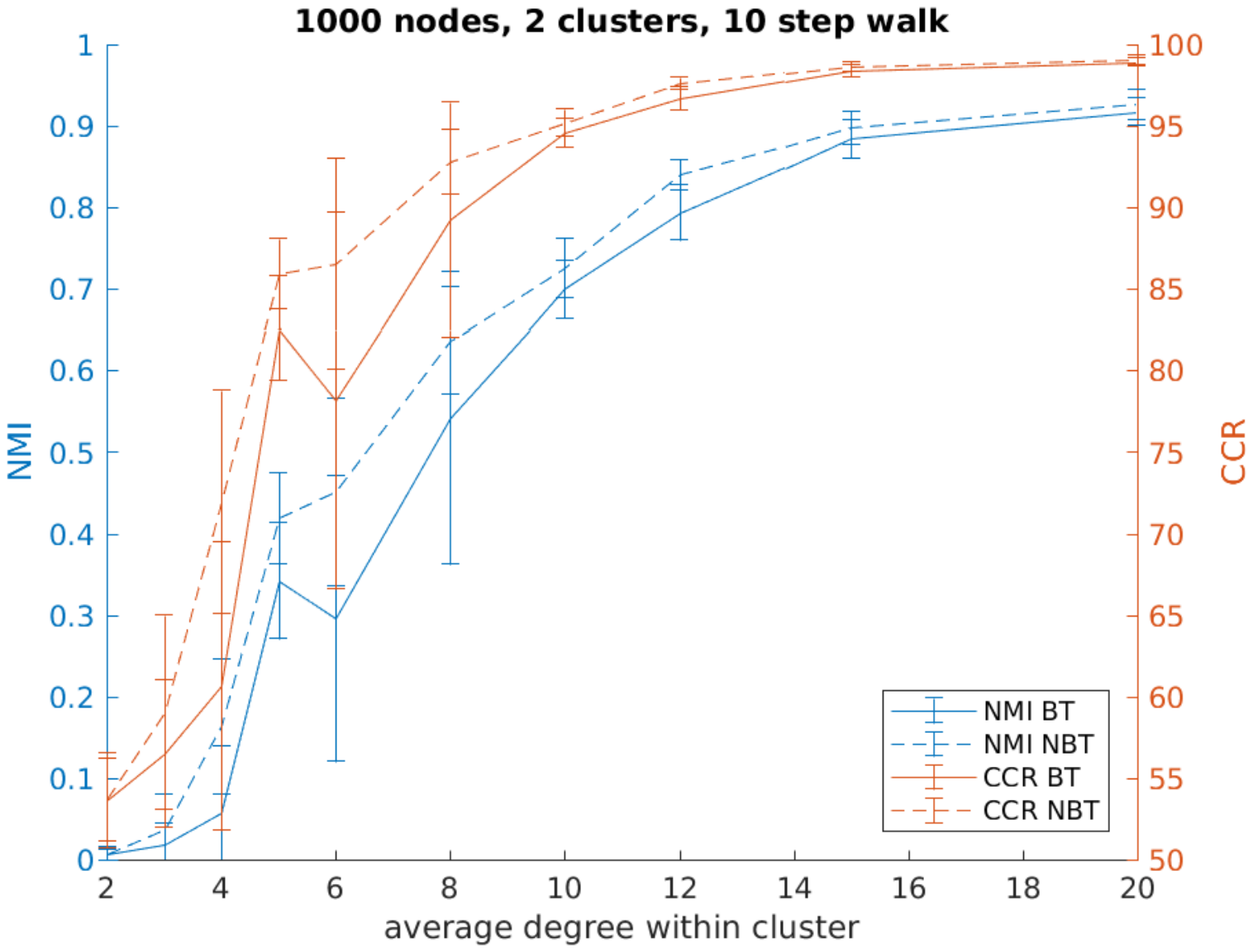} 
        \caption{1000 nodes}
        \label{fig:subim11}
    \end{subfigure}
    \hfill
    \begin{subfigure}{0.48\textwidth}
        \includegraphics[width=\linewidth]{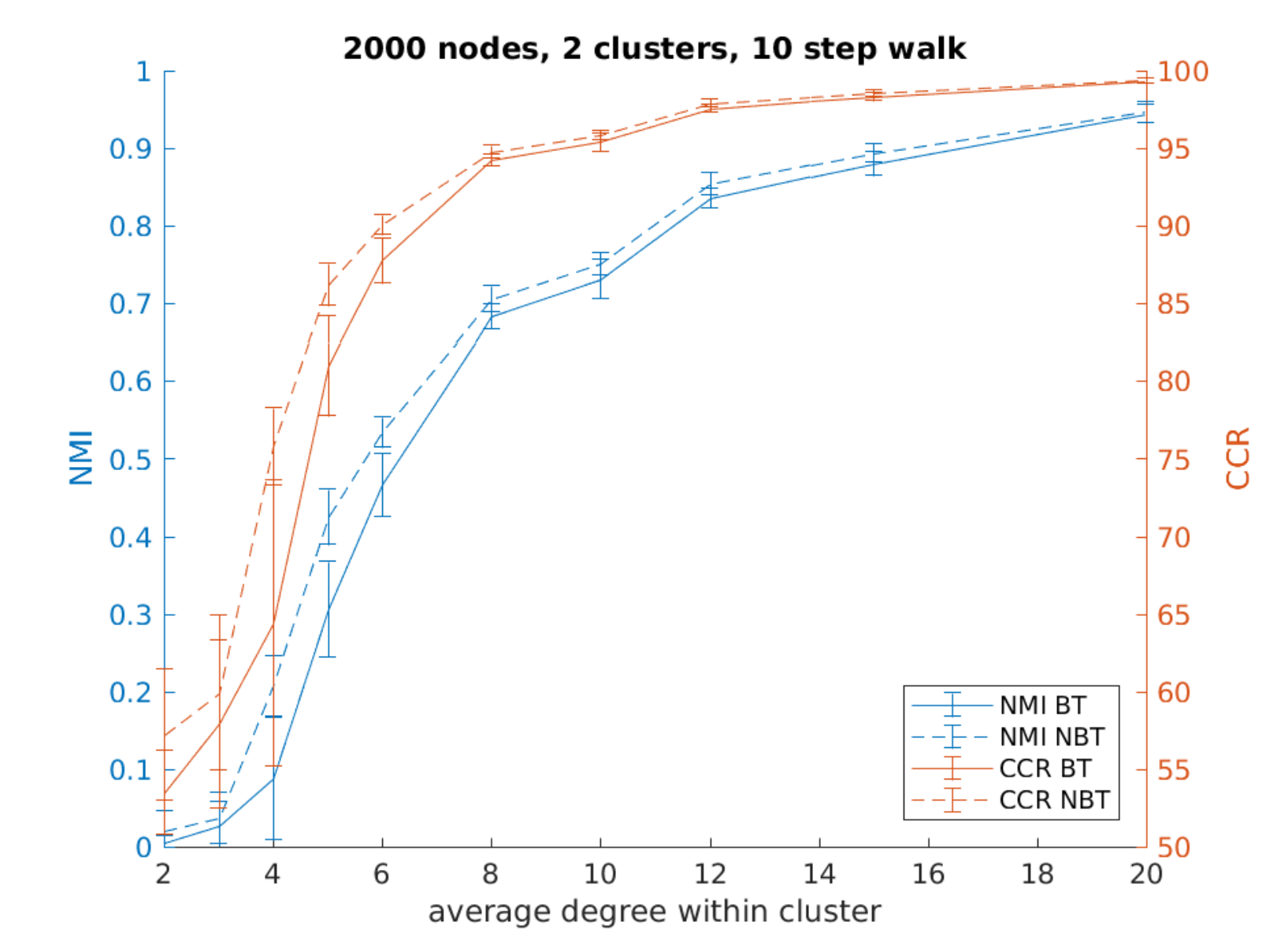}
        \caption{2000 nodes}
        \label{fig:subim12}
    \end{subfigure}
    \vskip \baselineskip
    \begin{subfigure}{0.48\textwidth}
        \centering
        \includegraphics[width=\linewidth]{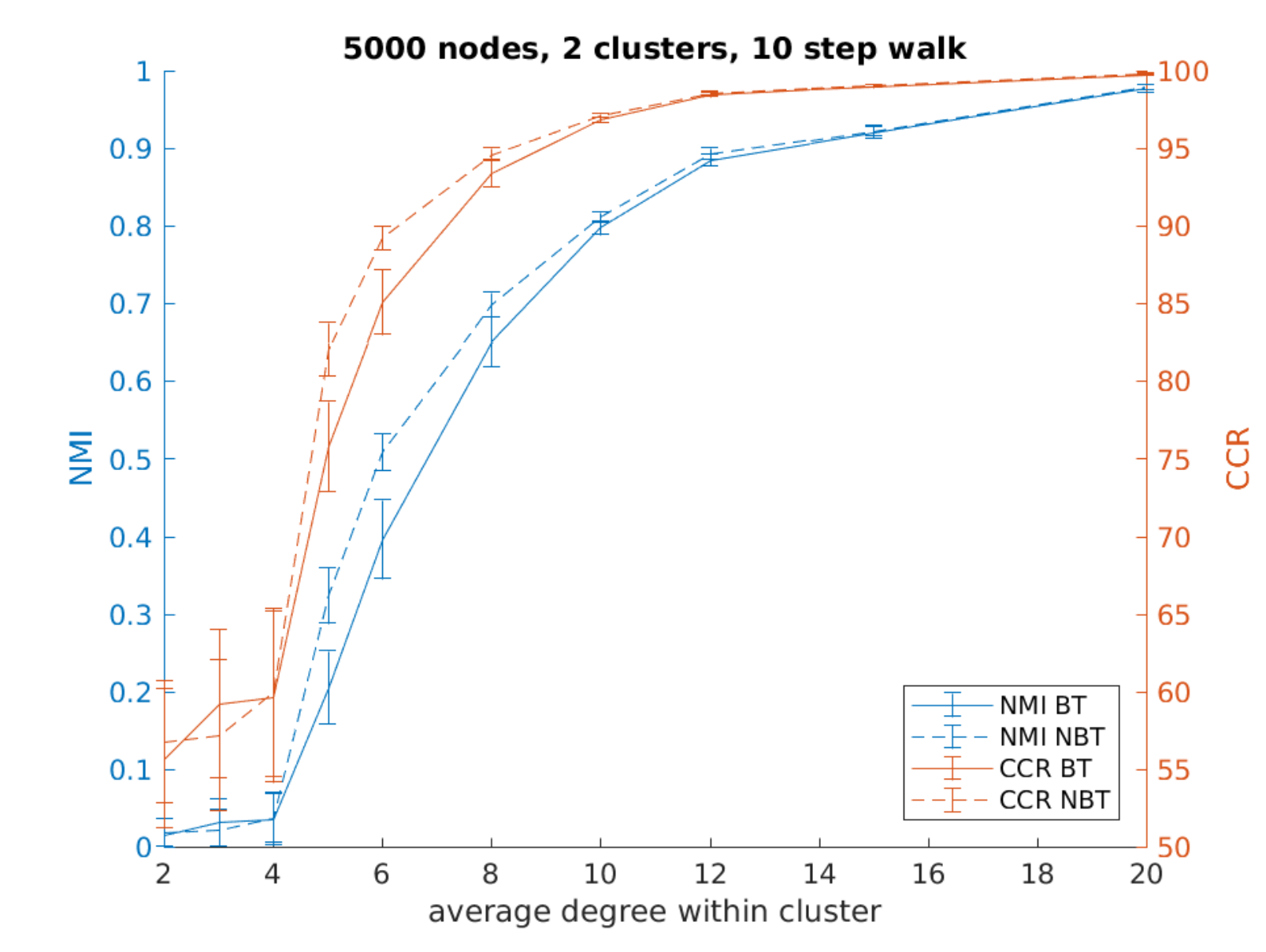}
        \caption{5000 nodes}
        \label{fig:subim13}
    \end{subfigure}
    \hfill
    \begin{subfigure}{0.48\textwidth}
        \centering
        \includegraphics[width=\linewidth]{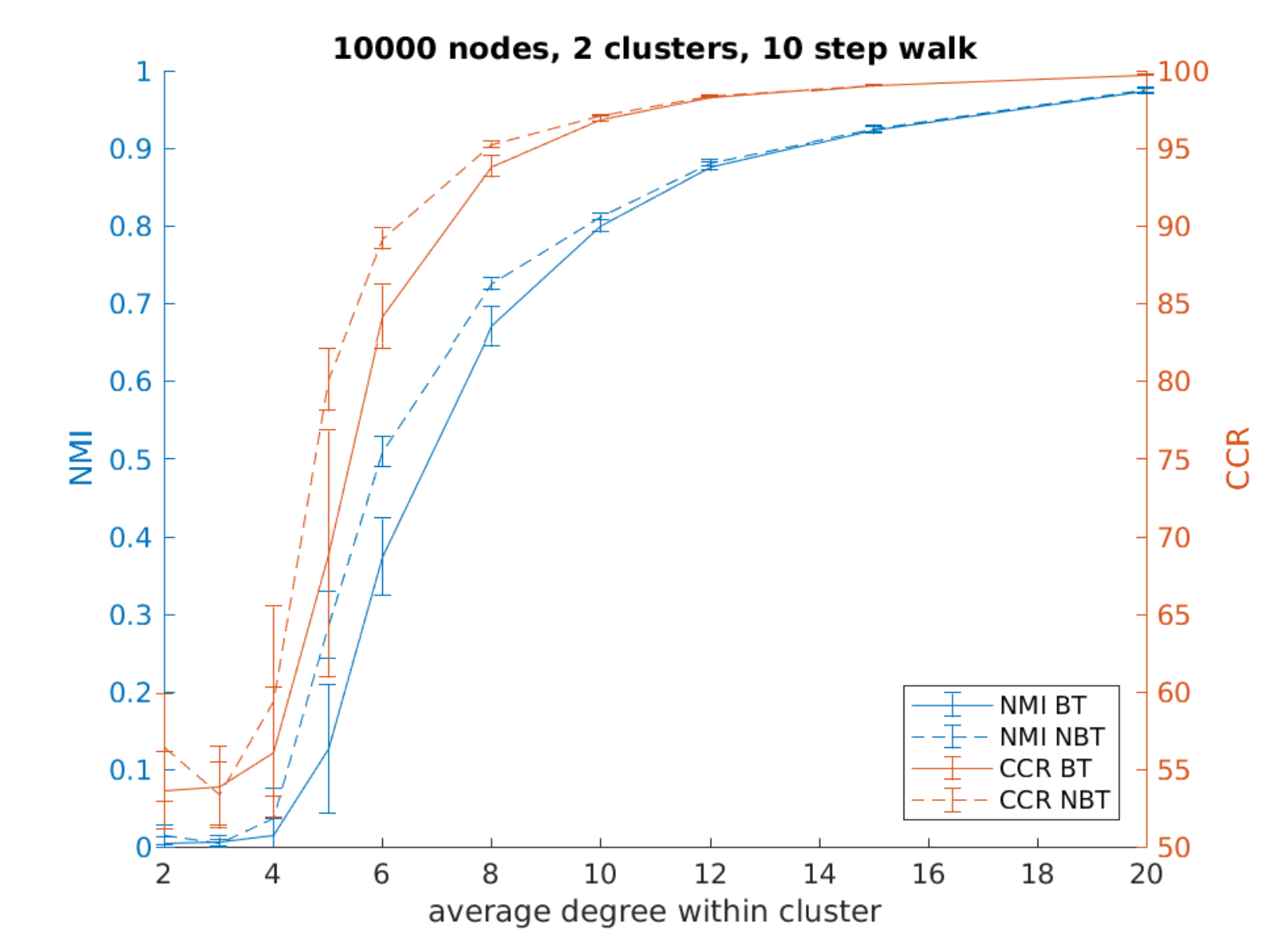}
        \caption{10000 nodes}
        \label{fig:subim14}
    \end{subfigure}
    \caption{\emph{Performance of both algorithms as a function of sparsity. We show performance for four differently-sized graphs. Note that measurement performance is noticeably better for VEC-NBT than for VEC.}}
    \label{fig:image1}
\end{figure}

\begin{figure}[H] %% Figure 2
        \begin{subfigure}{0.33\textwidth}
        \includegraphics[width=\linewidth]{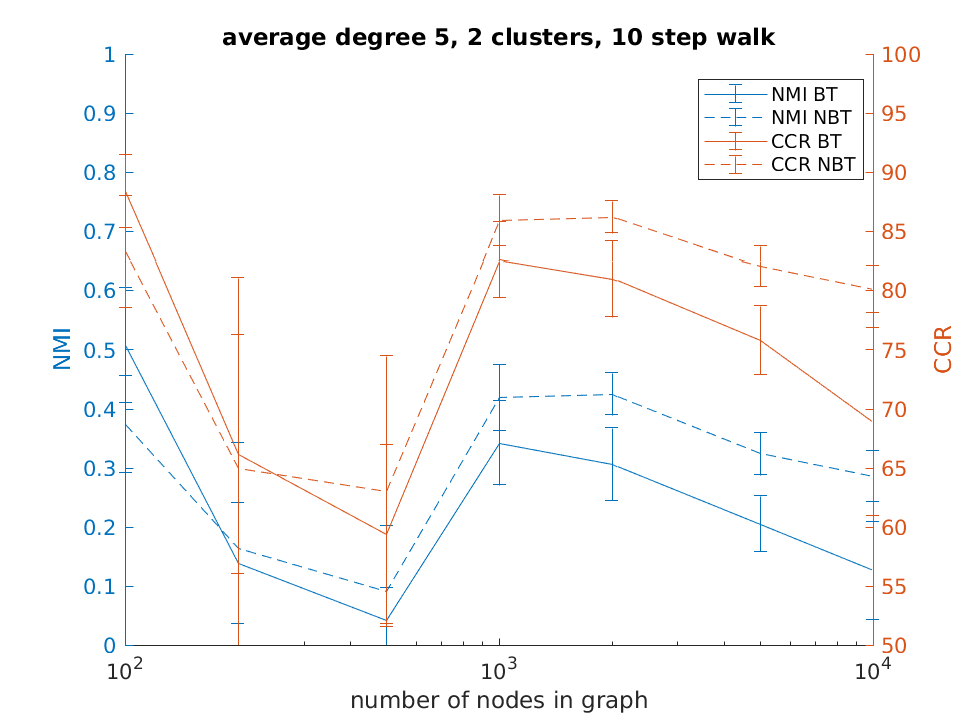} 
        \caption{average degree 5}
        \label{fig:subim21}
    \end{subfigure}
    \begin{subfigure}{0.33\textwidth}
        \includegraphics[width=\linewidth]{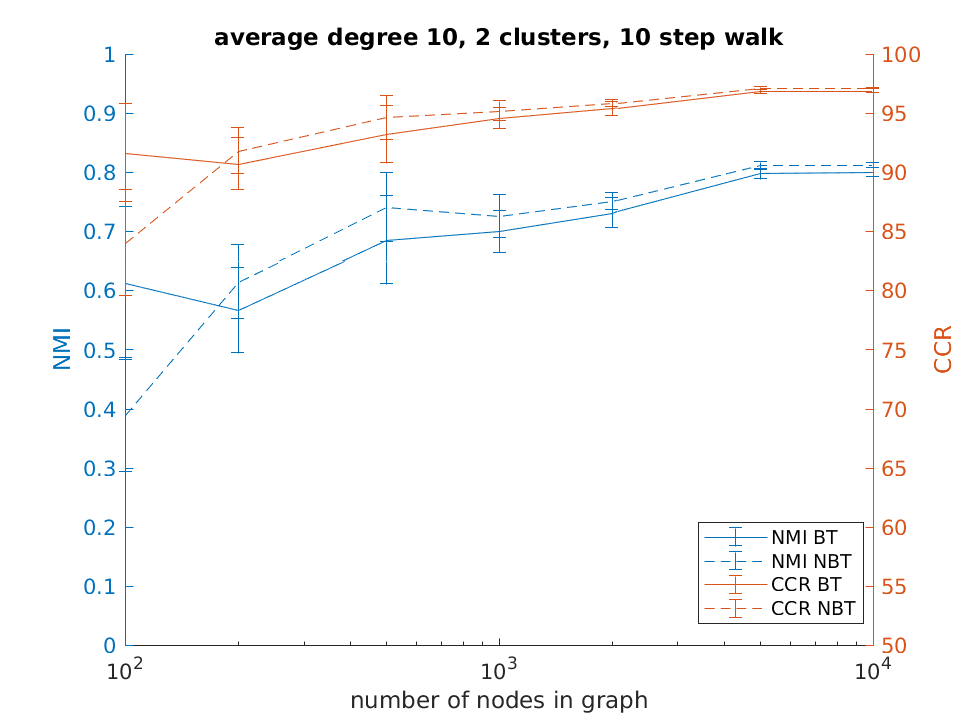}
        \caption{average degree 10}
        \label{fig:subim22}
    \end{subfigure}
    \begin{subfigure}{0.33\textwidth}
        \centering
        \includegraphics[width=\linewidth]{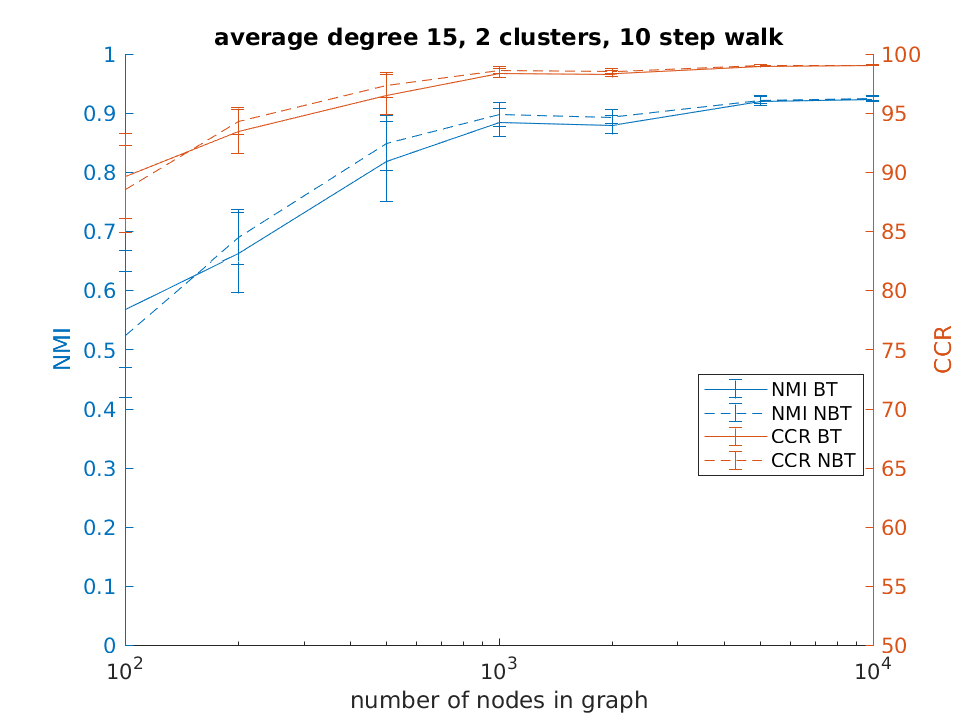}
        \caption{average degree 15}
        \label{fig:subim23}
    \end{subfigure}
    \caption{\emph{Performance of both algorithms as the number of nodes increases from 100 to 10000 for three graphs of different sparsity. The x-axis here is the number of nodes and the average degree is constant for each graph.}}
    \label{fig:image2}
\end{figure}

\vskip \baselineskip
\vskip \baselineskip

\begin{figure}[H] %% Figure 3
        \begin{subfigure}{0.33\textwidth}
        \includegraphics[width=\linewidth]{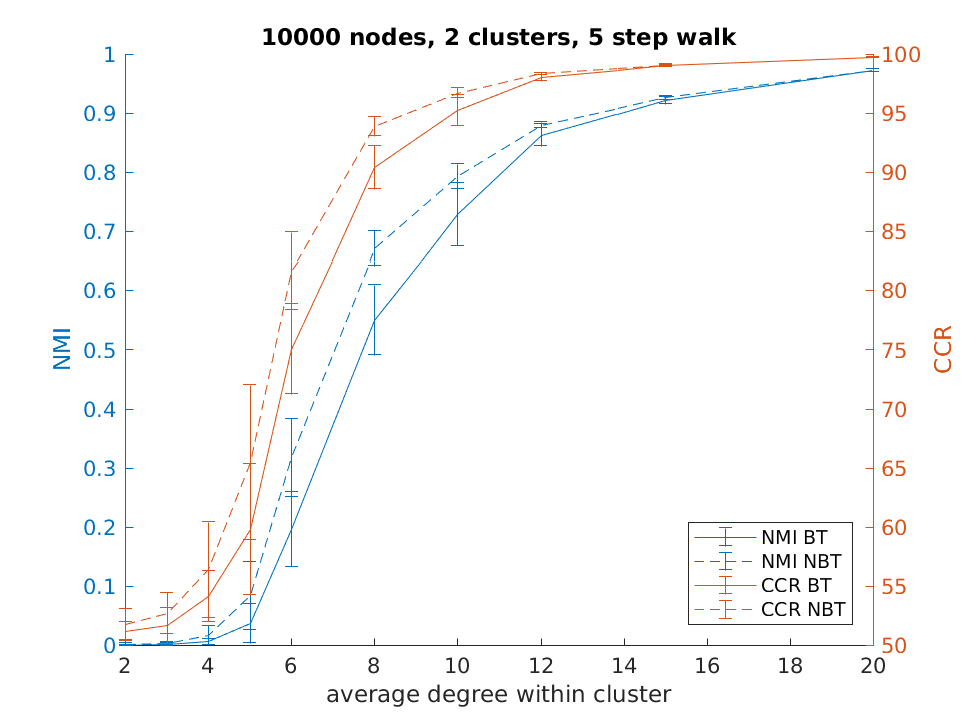} 
        \caption{5 step random walks}
        \label{fig:subim31}
    \end{subfigure}
    \begin{subfigure}{0.33\textwidth}
        \includegraphics[width=\linewidth]{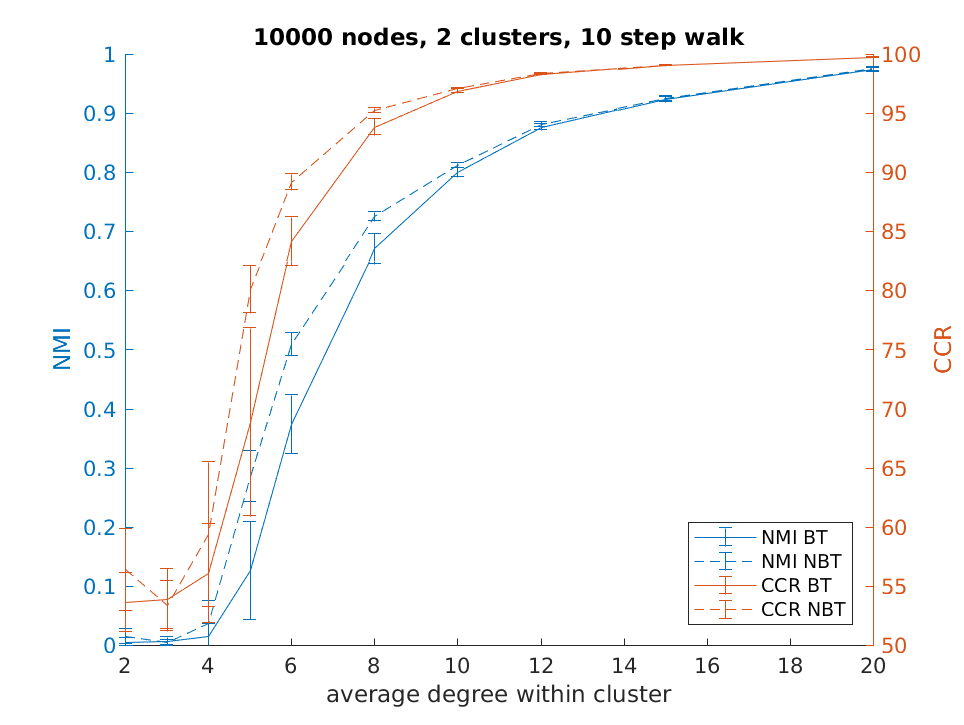}
        \caption{10 step random walks}
        \label{fig:subim32}
    \end{subfigure}
    \begin{subfigure}{0.33\textwidth}
        \includegraphics[width=\linewidth]{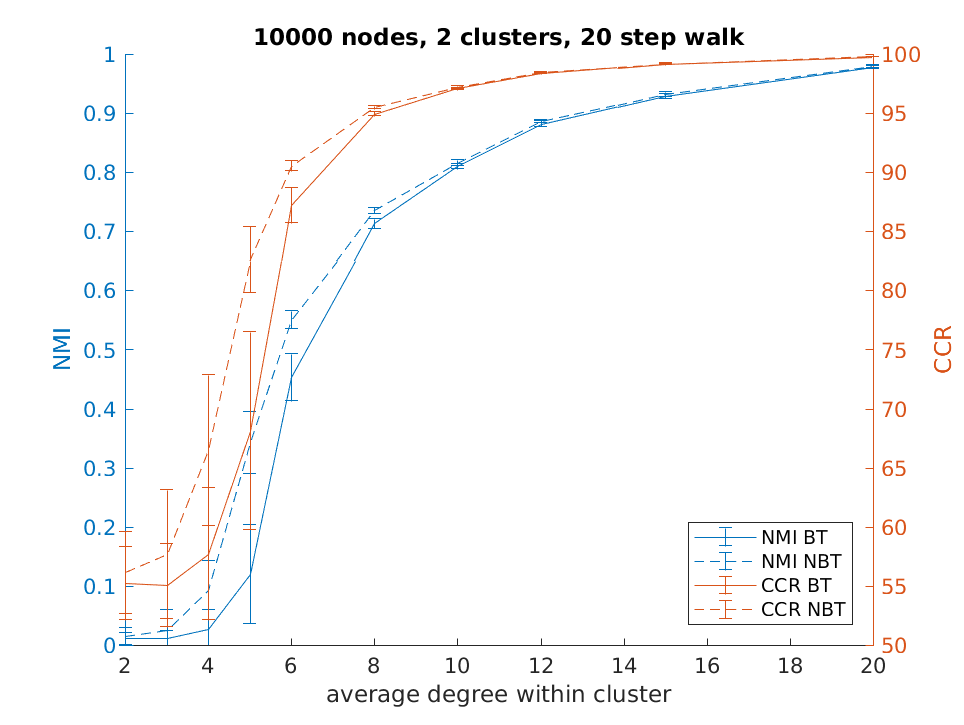}
        \caption{20 step random walks}
        \label{fig:subim33}
    \end{subfigure}
    \caption{\emph{Performance of both algorithms as the random walks increase in length, from 5 to 20 steps, plotting as a function of sparsity, showing that the curve moves rightward as the length of the walks increase. The improvement between backtracking and non-backtracking random walks is more visible for shorter walks.}}
    \label{fig:image3}
\end{figure}

\vskip \baselineskip
\vskip \baselineskip

\begin{figure}[H] %% Figure 4
    \begin{subfigure}{0.33\textwidth}
        \includegraphics[width=\linewidth]{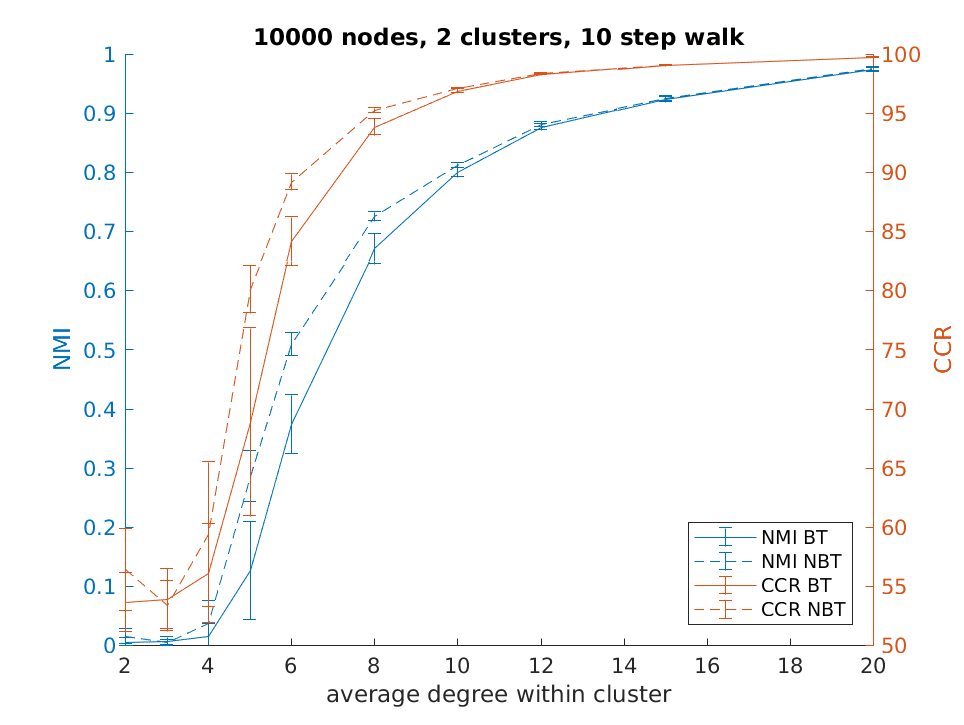}
        \caption{2 clusters}
        \label{fig:subim41}
    \end{subfigure}
    \begin{subfigure}{0.33\textwidth}
        \includegraphics[width=\linewidth]{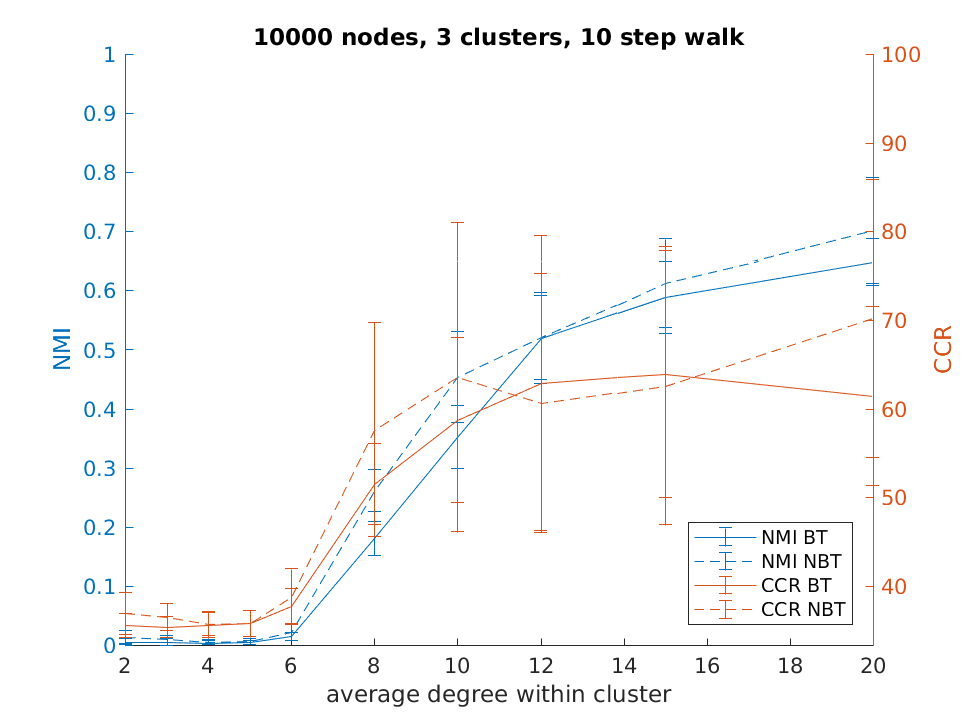}
        \caption{3 clusters}
        \label{fig:subim42}
    \end{subfigure}
    \begin{subfigure}{0.33\textwidth}
        \includegraphics[width=\linewidth]{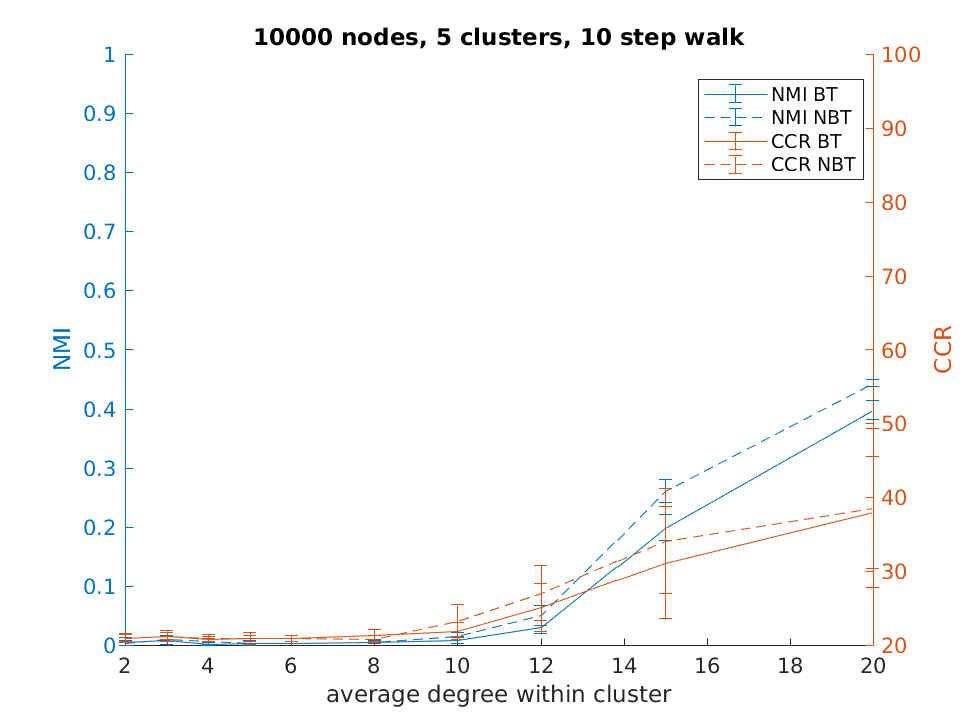}
        \caption{5 clusters}
        \label{fig:subim43}
    \end{subfigure}
    \caption{\emph{Performance of both algorithms as the number of clusters increases from 2 to 5. Performance is drastically lowered on anything larger than 2 graphs, but the improvement of using non-backtracking random walks is still clearly visible.}}
    \label{fig:image4}
\end{figure}

\vskip \baselineskip
\vskip \baselineskip

\section{Conclusion}
In this paper, we presented VEC-NBT, which is a modification of VEC \cite{NodeEmbed} using non-backtracking random walks instead of simple random walks. We show experimentally that VEC-NBT outperforms VEC for SBM model graphs across all ranges of parameters such as the number of nodes, number of clusters, and the length of the random walk, especially for sparser graphs. We discuss the theoretical basis for these results - the faster mixing rate of non-backtracking random walks compared to backtracking random walks. Finally, we analyze the connection between the begrudgingly-backtracking random walk we used and the non-backtracking random walk on which it was based. Future work will focus on formalizing our hypotheses about the behaviour of these two random walks, as well as exploring performance with other graph models which might better relate to real data.

\bibliographystyle{apalike}
\bibliography{biblio.bib}

\end{document}